%% file: birkhoff.tex
\documentclass[twoside]{article}

\input{preamble.tex}

\input{preamble_math.tex}

\input{preamble_acronyms.tex}

\usepackage{blindtext}

\DeclareRobustCommand{\parhead}[1]{\textbf{#1}~}

\usepackage[accepted]{aistats2018}
\begin{document}

\runningtitle{Reparameterizing the Birkhoff Polytope for
  Variational Permutation Inference}

\runningauthor{Linderman, Mena, Cooper, Paninski, and Cunningham}

\twocolumn[

\aistatstitle{Reparameterizing the Birkhoff Polytope for \\
  Variational Permutation Inference}

\aistatsauthor{
  Scott W. Linderman$^*$
  \And Gonzalo E. Mena$^*$
  \And  Hal Cooper}
\aistatsaddress{ Columbia University \And Columbia University \And Columbia University}
\aistatsauthor{Liam Paninski \And John P. Cunningham }
\aistatsaddress{Columbia University \And Columbia University}

%\aistatsauthor{ Anonymous Authors }

%\aistatsaddress{ Anonymous Institutions}
]

\begin{abstract}
  Many matching, tracking, sorting, and ranking problems require
  probabilistic reasoning about possible permutations, a set that
  grows factorially with dimension. Combinatorial optimization
  algorithms may enable efficient point estimation, but fully Bayesian
  inference poses a severe challenge in this high-dimensional,
  discrete space.  To surmount this challenge, we start with the usual
  step of relaxing a discrete set (here, of permutation matrices) to
  its convex hull, which here is the Birkhoff polytope: the set of all
  doubly-stochastic matrices.  We then introduce two novel
  transformations: first, an invertible and differentiable
  stick-breaking procedure that maps unconstrained space to the
  Birkhoff polytope; second, a map that rounds points toward the
  vertices of the polytope.  Both transformations include a temperature
  parameter that, in the limit, concentrates the densities on
  permutation matrices.  We then exploit these transformations and
  reparameterization gradients to introduce variational inference over
  permutation matrices, and we demonstrate its utility in a series of 
  experiments.
\end{abstract}

\section{Introduction}
{\let\thefootnote\relax\footnote{$^*$These authors contributed equally.}}

Permutation inference is central to many modern machine learning
problems.  Identity management ~\citep{guibas2008identity} and
multiple-object tracking~\citep{shin2005lazy, kondor2007multi} are
fundamentally concerned with finding a permutation that maps an
observed set of items to a set of canonical labels.  Ranking problems,
critical to search and recommender systems, require inference over the
space of item orderings \citep{meilua2007consensus, lebanon2008non,
  adams2011ranking}.  Furthermore, many probabilistic models, like
preferential attachment network models~\citep{bloem2016random} and
repulsive point process models~\citep{rao2016bayesian}, incorporate a
latent permutation into their generative processes; inference over
model parameters requires integrating over the set of permutations
that could have given rise to the observed data.  In neuroscience,
experimentalists now measure whole-brain recordings in
\textit{C. Elegans}~\citep{Kato2015, nguyen2016whole}, a model
organism with a known synaptic network~\citep{white1986structure}; a
current challenge is matching the observed neurons to corresponding
nodes in the reference network.  In Section~\ref{sec:celegans}, we
address this problem from a Bayesian perspective in which permutation
inference is a central component of a larger inference problem involving
unknown model parameters and hierarchical structure.

% Emphasize the importance of Bayesian approach and recent advances
% in variational inference
The task of computing optimal point estimates of permutations under
various loss functions has been well studied in the combinatorial
optimization literature ~\citep{kuhn1955hungarian,
  munkres1957algorithms, lawler1963quadratic}. However, many
probabilistic tasks, like the aforementioned neural identity inference
problem, require reasoning about the posterior distribution over
permutation matrices.  A variety of Bayesian permutation inference
algorithms have been proposed, leveraging sampling methods
\citep{diaconis1988group, miller2013exact, harrison2013importance},
Fourier representations~\citep{kondor2007multi, huang2009fourier}, as
well as convex~\citep{lim2014beyond} and
continuous~\citep{plis2011directional} relaxations for approximating
the posterior distribution.  Here, we address this problem from an
alternative direction, leveraging stochastic variational
inference~\citep{hoffman2013stochastic} and reparameterization
gradients~\citep{rezende2014stochastic, Kingma2014} to derive a
scalable and efficient permutation inference algorithm.

% Paper structure
Section~\ref{sec:background} lays the necessary groundwork,
introducing definitions, prior work on permutation inference,
variational inference, and continuous relaxations.
Section~\ref{sec:permutation} presents our primary contribution: a
pair of transformations that enable variational inference over
doubly-stochastic matrices, and, in the zero-temperature limit,
permutations, via stochastic variational inference.  In the process,
we show how these transformations connect to recent work on discrete
variational inference~\citep{maddison2016concrete,
  jang2016categorical, balog2017lost}.  Sections~\ref{sec:synthetic}
and~\ref{sec:celegans} present a variety of experiments that
illustrate the benefits of the proposed variational approach.
Further details are in the supplement.
  
\section{Background}
\label{sec:background}

% We begin with definitions and notation, a review of
% variational inference and the reparameterization trick, and a
% discussion related work.

\subsection{Definitions and notation.}  A permutation is a bijective
mapping of a set onto itself.  When this set is finite, the mapping is
conveniently represented as a binary
matrix~${X \in \{0,1\}^{N \times N}}$ where~${X_{m,n}=1}$ implies that
element~$m$ is mapped to element~$n$.  Since permutations are
bijections, both the rows and columns of~$X$ must sum to one.  From a
geometric perspective, the Birkhoff-von Neumann theorem states that
the convex hull of the set of permutation matrices is the set of
doubly-stochastic matrices; i.e. non-negative square matrices whose
rows and columns sum to one. The set of doubly-stochastic matrices is
known as the \emph{Birkhoff polytope}, and it is defined by,
\begin{align*}
  \mcB_N = \Big \{X : \qquad 
           X_{m,n} &\geq 0   & &\forall \, m,n \in 1, \ldots, N; \\
           \sum_{n=1}^N X_{m,n} &= 1  & &\forall \, m \in 1, \ldots, N; \\
           \sum_{m=1}^N X_{m,n} & =1 &  &\forall \, n \in 1, \ldots, N \Big\}.
\end{align*}
These linear row- and column-normalization constraints
restrict~$\mcB_N$ to a~${(N-1)^2}$ dimensional subset
of~$\reals^{N \times N}$.  Despite these constraints, we have a number
of efficient algorithms for working with these objects.  The
\emph{Sinkhorn-Knopp algorithm}~\citep{sinkhorn1967concerning}
maps the positive orthant onto~$\mcB_N$ by iteratively normalizing
the rows and columns, and the \emph{Hungarian
  algorithm}~\citep{kuhn1955hungarian, munkres1957algorithms} solves
the minimum weight bipartite matching problem---optimizing a linear
objective over the set of permutation matrices---in cubic time.

\subsection{Related Work}
A number of previous works have considered approximate methods of
posterior inference over the space of permutations.  When a point
estimate will not suffice, sampling methods like Markov chain Monte
Carlo (MCMC) algorithms may yield a reasonable approximate posterior
for simple problems~\citep{diaconis1988group}.
\citet{harrison2013importance} developed an importance sampling
algorithm that fills in count matrices one row at a time, showing
promising results for matrices with~$O(100)$ rows and
columns. \citet{li2013efficient} considered using the Hungarian
algorithm within a Perturb-and-MAP algorithm for approximate sampling.
Another line of work considers inference in the spectral domain,
approximating distributions over permutations with the low frequency
Fourier components~\citep{kondor2007multi, huang2009fourier}.  Perhaps
most relevant to this work, \citet{plis2011directional} propose a
continuous relaxation from permutation matrices to points on a
hypersphere, and then use the von Mises-Fisher (vMF) distribution to
model distributions on the sphere's surface.
% While the vMF
% distribution does have a concentration parameter, as the concentration
% goes to infinity, the distribution converges to a point on the sphere.
We will relax permutations to points in the Birkhoff polytope and
derive temperature-controlled densities such that as the
temperature goes to zero, the distribution converges to an atomic
density on permutation matrices.  This will enable efficient variational
inference with the reparameterization trick, which we describe next.

\subsection{Variational inference and the reparameterization trick}
\label{sub:repa}
Given an intractable model with data~$y$, likelihood~$p(y \given x)$,
and prior~$p(x)$, variational Bayesian inference algorithms aim to
approximate the posterior distribution~$p(x \given y)$ with a more
tractable distribution~$q(x; \theta)$, where ``tractable'' means that,
at a minimum, we can sample~$q$ and evaluate it pointwise (including
its normalization constant)~\citep{Blei2017}.  We find this
approximate distribution by searching for the parameters~$\theta$ that
minimize the Kullback-Leibler (KL) divergence between~$q$ and the true
posterior, or equivalently, maximize the evidence lower bound~(ELBO),
\begin{align*}
  \mcL(\theta) &\triangleq \bbE_q \left[ \log p(x, y) - \log q(x; \theta) \right].
\end{align*}
Perhaps the simplest method of optimizing the ELBO is stochastic
gradient ascent.  However, computing~$\nabla_\theta \mcL(\theta)$
requires some care since the ELBO contains an expectation with respect
to a distribution that depends on these parameters.

\begin{figure*}[ht!]
  \centering
  \includegraphics[width=6.5in]{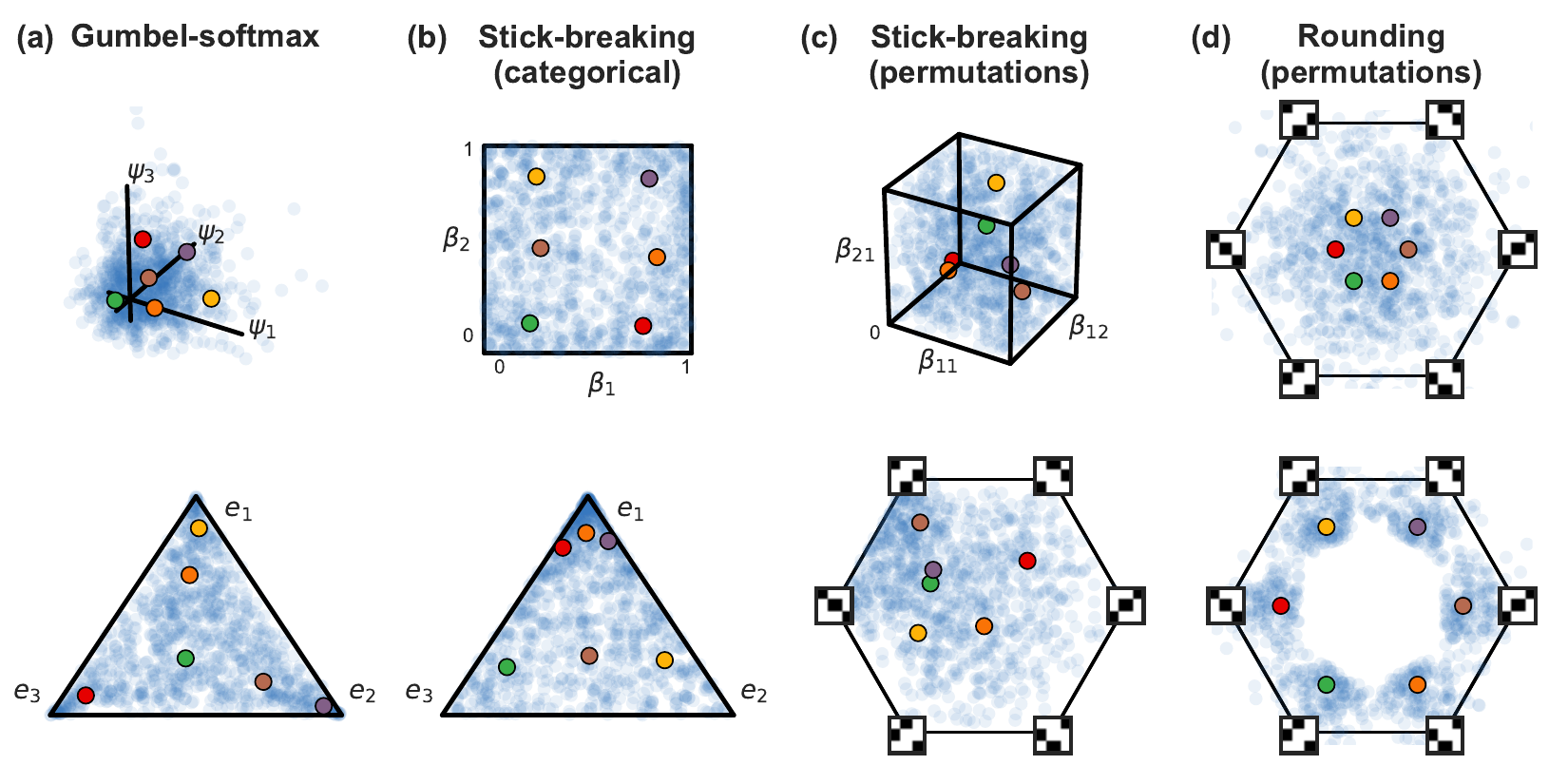} 
  \caption{Reparameterizations of discrete polytopes.  From left to
    right: (a)~The Gumbel-softmax, or ``Concrete'' transformation maps
    Gumbel r.v.'s~${\psi \in \reals^N}$ (blue dots) to points in the
    simplex~${x \in \Delta_{N}}$ by applying the softmax.  Colored
    dots are random variates that aid in visualizing the
    transformation.  (b)~Stick-breaking offers and alternative
    transformation for categorical inference, here from
    points~$\beta \in [0,1]^{N-1}$ to~$\Delta_N$, but the ordering of
    the stick-breaking induces an asymmetry in the transformation.
    (c)~We extend this stick-breaking transformation to reparameterize
    the Birkhoff polytope, i.e. the set of doubly-stochastic
    matrices. We show how~$\mcB_3$ is reparameterized in terms of
    matrices~$B \in [0,1]^{2 \times 2}$ These points are mapped to
    doubly-stochastic matrices, which we have projected
    onto~$\reals^2$ below (stencils show permutation matrices at the
    vertices).  (d)~Finally, we derive a ``rounding'' transformation
    that moves points in~$\reals^{N \times N}$ nearer to the closest
    permutation matrix, which is found with the Hungarian algorithm.
    This is more symmetric, but does not map strictly onto~$\mcB_N$.
  }
\label{fig:transforms}
\end{figure*}

When~$x$ is a continuous random variable, we can sometimes leverage the \emph{reparameterization trick}
\citep{Salimans2013, Kingma2014}.  Specifically, in some cases we can
simulate from~$q$ via the following equivalence,
\begin{align*}
  x &\sim q(x; \theta)
      & \iff & &  
  \noise &\sim r(\noise), \quad x = g(\noise;\theta),
\end{align*}
where~$r$ is a distribution on the ``noise''~$\noise$ and
where~$g(\noise; \theta)$ is a deterministic and differentiable
function.
% For example,
% if~${q(x; \theta) = \mathcal{N}(x \given \theta, 1)}$, we can
% reparameterize by setting the noise distribution
% to~${r(\noise) = \mathcal{N}(\noise \given 0, 1)}$ and using the
% transformation~${g(\noise; \theta) = \noise + \theta}$.
The reparameterization trick effectively ``factors out'' the randomness
of~$q$. With this transformation, we can bring the gradient inside the
expectation as follows,
\begin{multline}
  \label{eq:elbo}
  \nabla_\theta \mcL(\theta) 
  = \E_{r(\noise)} \Big[ \nabla_\theta \log p(g(\noise; \theta) \given y) \\
    - \nabla_\theta  \log q(g(\noise; \theta); \theta) \Big].
\end{multline}
This gradient can be estimated with Monte Carlo, and, in practice,
this leads to lower variance estimates of the gradient than, for
example, the score function estimator \citep{Williams1992, Glynn1990}.

Critically, the gradients in~\eqref{eq:elbo} can only be computed
if~$x$ is continuous. Recently, \citet{maddison2016concrete} and
\citet{jang2016categorical} proposed the ``Gumbel-softmax'' method for
discrete variational inference. It is based on the following
observation: discrete probability mass functions~$q(x; \theta)$ can be
seen as densities with atoms on the vertices of the simplex; i.e. on
the set of one-hot vectors~${\{e_n\}_{n=1}^N}$,
where~${e_n = (0, 0, \ldots, 1, \ldots, 0)^\trans}$ is a length-$N$ binary vector
with a single 1 in the~$n$-th position. This motivates a natural
relaxation: let $q(x; \theta)$ be a density on the interior of the
simplex instead, and anneal this density such that it converges to an
atomic density on the vertices. Fig.~\ref{fig:transforms}a illustrates
this idea. Gumbel random variates, are mapped through a
temperature-controlled softmax function,
${g_\tau(\psi) = \big[e^{\psi_1 / \tau}/Z, \ldots, e^{\psi_N / \tau}/Z
  \big]}$, where~${Z=\sum_{n=1}^N e^{\psi_n / \tau}}$, to obtain
points in the simplex. As~$\tau$ goes to zero, the density
concentrates on one-hot vectors.  We build on these ideas for
variational permutation inference.

\section{Variational permutation inference via reparameterization}
\label{sec:permutation}
The Gumbel-softmax method scales linearly with the support of the
discrete distribution, rendering it prohibitively expensive for direct
use on the set of~$N!$ permutations.  Instead, we develop two
transformations to map~$O(N^2)$-dimensional random variates to points
in or near the Birkhoff polytope.\footnote{While Gumbel-softmax does
  not immediately extend to permutation inference, the methods
  presented herein easily extend to categorical inference.  We
  explored this direction experimentally and show results in the
  supplement.}  Like the Gumbel-softmax method, these transformations
will be controlled by a temperature that concentrates the resulting
density near permutation matrices.  The first method is a novel
``stick-breaking'' construction; the second rounds points toward
permutations with the Hungarian algorithm.  We present these in turn
and then discuss their relative merits. We provide further
implementation details for both methods in the supplement.

\subsection{Stick-breaking transformations to the Birkhoff polytope}
\label{sub:stickbreaking}

Stick-breaking is well-known as a construction for the Dirichlet
process~\citep{sethuraman1994constructive}; here we show how the
same intuition can be extended to more complex discrete objects. 
Let~$B$ be a matrix in~${[0,1]^{(N-1) \times (N-1)}}$; we will
transform it into a doubly-stochastic
matrix~${X \in [0,1]^{N \times N}}$ by filling in entry by entry, starting
in the top left and raster scanning left to right then top to
bottom. Denote the~$(m,n)$-th entries of~$B$ and~$X$ by~$\beta_{mn}$
and~${x}_{mn}$, respectively.

Each row and column has an associated unit-length ``stick'' that we
allot to its entries.  The first entry in the matrix is given by
$x_{11} = \beta_{11}$.  As we work left to right in the first row, the
remaining stick length decreases as we add new entries. This reflects
the row normalization constraints.  The first row follows the standard
stick-breaking construction,
\begin{align*}
  x_{1n} &= \beta_{1n} \left(1 - \sum_{k=1}^{n-1} x_{1k} \right)  & &  \text{for } n=2, \ldots, N-1\\
  x_{1N} &= 1 - \sum_{n=1}^{N-1} x_{1n}.
\end{align*}
This is illustrated in Fig.~\ref{fig:transforms}b, where points in the
unit square map to points in the simplex. Here, the blue dots are
two-dimensional~$\distNormal(0, 4I)$ variates mapped through a coordinate-wise
logistic function.

Subsequent rows are more interesting, requiring a novel advance on the
typical uses of stick breaking. Here we need to conform to row and
column sums (which introduce upper bounds), and a lower bound
induced by stick remainders that must allow completion of subsequent
sum constraints.  Specifically, the remaining rows must now conform to
both row- and column-constraints. That is,
\begin{align*}
x_{mn} &\leq 1- \sum_{k=1}^{n-1} x_{mk} & & \text{(row sum)} \\
x_{mn} &\leq 1- \sum_{k=1}^{m-1} x_{kn} & & \text{(column sum)}.
\end{align*}
Moreover, there is also a lower bound on~$x_{mn}$. This entry must
claim enough of the stick such that what is leftover fits within
the confines imposed by subsequent column sums. That is, each column
sum places an upper bound on the amount that may be attributed to any
subsequent entry. If the remaining stick exceeds the sum of these
upper bounds, the matrix will not be doubly-stochastic.  Thus,
\begin{align*}
\underbrace{1 - \sum_{k=1}^n x_{mk}}_{\text{remaining stick}}
  &\leq \underbrace{\sum_{j=n+1}^N (1- \sum_{k=1}^{m-1} x_{kj})}_{
    \text{remaining upper bounds}}.
\end{align*}
Rearranging terms, we have,
\begin{align*}
  x_{mn} &\geq
  % 1- \sum_{k=1}^{n-1} x_{mk} - \sum_{j=n+1}^N (1- \sum_{k=1}^{m-1} x_{kj}) \\
1 - N + n - \sum_{k=1}^{n-1} x_{mk}  +  \sum_{k=1}^{m-1} \sum_{j=n+1}^N x_{kj}.
\end{align*}
Of course, this bound is only relevant if the right hand side is greater than zero.
Taken together, we have~$\ell_{mn} \leq x_{mn} \leq u_{mn}$, where,
\begin{align*}
\ell_{mn} &\triangleq \max \left \{0, \, 1 - N + n - \sum_{k=1}^{n-1} x_{mk}  +  \sum_{k=1}^{m-1} \sum_{j=n+1}^N x_{kj} \right \}
\\
u_{mn} &\triangleq 
\min \left \{1- \sum_{k=1}^{n-1} x_{mk}, \,
1- \sum_{k=1}^{m-1} x_{kn} \right\}.
\end{align*}
Accordingly, we define~${x_{mn} = \ell_{mn} + \beta_{mn} (u_{mn} - \ell_{mn})}$.
The inverse transformation from~$X$ to $B$ is analogous.
We start by computing~$z_{11}$ and then progressively compute
upper and lower bounds and set~${\beta_{mn} = (x_{mn} - \ell_{mn})/(u_{mn} - \ell_{mn})}$.

To complete the reparameterization, we define a parametric,
temperature-controlled density from a standard Gaussian matrix~${\Noise \in \reals^{(N-1) \times (N-1)}}$
to the unit-hypercube~$B$.
Let,
\begin{align*}
  \psi_{mn} &= \mu_{mn} + \nu_{mn} \noise_{mn}, \\
   \beta_{mn} &= \sigma\left( \psi_{mn} / \tau \right),
\end{align*}
where~${\theta = \{\mu_{mn}, \nu^2_{mn}\}_{m,n=1}^N}$ are the mean and
variance parameters of the intermediate Gaussian
matrix~$\Psi$,~${\sigma(u) = (1+e^{-u})^{-1}}$ is the logistic
function, and~$\tau$ is a temperature parameter. As~$\tau \to 0$, the
values of~$\beta_{mn}$ are pushed to either zero or one, depending on
whether the input to the logistic function is negative or positive,
respectively.  As a result, the doubly-stochastic output matrix~$X$ is
pushed toward the extreme points of the Birkhoff polytope, the
permutation matrices.  This map is illustrated in
Fig.~\ref{fig:transforms}c for permutations of~${N=3}$ elements.
Here, the blue dots are samples of~$B$ with~$\mu_{mn}=0$,
$\nu_{mn}=2$, and~$\tau=1$.

We compute gradients of this transformation with automatic
differentiation.  Since this transformation is ``feed-forward,'' its
Jacobian is lower triangular. The determinant of the Jacobian,
necessary for evaluating the density~$q_\tau(X; \theta)$, is a simple
function of the upper and lower bounds and is derived in
Appendix~\ref{sec:details}.  While this map is peculiar in its
reliance on an ordering of the elements, as discussed in
Section~\ref{sec:considerations}, it is a novel transformation to the
Birkhoff polytope that supports gradient-based variational
permutation inference.

\subsection{Rounding toward permutation matrices}
\label{sub:rounding}

While relaxing permutations to the Birkhoff polytope is intuitively
appealing, it is not strictly required.  For example, consider the
following procedure for sampling a point \emph{near} the Birkhoff
polytope:
\begin{enumerate}[label=(\roman*)]
\item Input~${\Noise \in \reals^{N \times N}}$,~${M \in \reals_+^{N \times N}}$, and~${V \in \reals_+^{N \times N}}$;
\item Map~$M \to \widetilde{M}$, a point in the Birkhoff polytope, using the Sinkhorn-Knopp algorithm;
\item Set~${\Psi = \widetilde{M} + V \odot \Noise}$ where~$\odot$ denotes elementwise multiplication;
\item Find~$\mathsf{round}(\Psi)$, the nearest permutation matrix to~$\Psi$, using the Hungarian algorithm;
\item Output~${X = \tau \Psi + (1-\tau) \mathsf{round}(\Psi)}$.
\end{enumerate}
This procedure defines a mapping~${X = g_\tau(\Noise; \theta)}$ with~${\theta = \{M, V\}}$. When the elements
of~$\Noise$ are independently sampled from a standard normal
distribution, it implicitly defines a distribution over matrices~$X$
parameterized by~${\theta}$. Furthermore, as~$\tau$ goes to
zero, the density concentrates on permutation matrices.  A simple
example is shown in Fig.~\ref{fig:transforms}d,
where~${M = \tfrac{1}{N}\bone \bone^\trans}$ with~$\bone$ a
vector of all ones,
${V = 0.4^2 \bone \bone^\trans}$, and~${\tau=0.5}$. We use this procedure to define a variational
distribution with density~$q_\tau(X; \theta)$.

To compute the ELBO and its gradient~\eqref{eq:elbo}, we need to
evaluate~$q_\tau(X; \theta)$.  By construction, steps (i) and (ii) involve
differentiable transformations of parameter~$M$ to set the mean close
to the Birkhoff polytope, but since these do not influence the
distribution of~$\Noise$, the non-invertibility of
the Sinkhorn-Knopp algorithm poses no problems.  Had
we applied this algorithm directly to~$\Noise$, this would not be true.
The challenge in computing the density stems from the rounding in
steps~(iv) and~(v).

To compute~$q_\tau(X; \theta)$, we need the
inverse~$g_\tau^{-1}(X; \theta)$ and its Jacobian.  The inverse is
straightforward: when~${\tau \in [0,1)}$, $\mathsf{round}(\Psi)$
outputs a point strictly closer to the nearest permutation, implying
${\mathsf{round}(\Psi) \equiv \mathsf{round}(X)}$.  Thus, the inverse
is~${g_\tau^{-1}(X; \theta) = \big(\tfrac{1}{\tau}X - \tfrac{1-\tau}{\tau}
  \mathsf{round}(X) - \widetilde{M}\big) \oslash V}$, where~$\oslash$ denotes elementwise division.  A slight wrinkle arises from the fact that
step (v) maps to a subset~${\mcX_\tau \subset \reals^{N \times N}}$ that
excludes the center of the Birkhoff polytope (note the ``hole'' in
Fig.~\ref{fig:transforms}d), but this inverse is valid for all~$X$ in
that subset.

% \footnote{Consider a simple example of rounding in the one-dimensional
%   simplex, that is, the unit interval.  If~$\tau = 0.5$, the rounding
%   operation maps~$[0,1]$ to~${[0,0.25) \cup [0.75, 1]}$; the resulting
%   density has zero measure in the interval~$[0.25, 0.75)$.  The same
%   is true of rounding toward permutations: the inverse mapping is only
%   defined for points within~$\tau$ of a permutation. }

The Jacobian is more challenging due to the non-differentiability
of~$\mathsf{round}$. However, since the nearest permutation output
only changes at points that are equidistant from two or more
permutation matrices, $\mathsf{round}$ is a piecewise constant
function with discontinuities only at a set of points with
zero measure. Thus, the change of variables theorem still applies.

With the inverse and its Jacobian, we have
\begin{align*}
  q_\tau(X; \theta) = 
  \prod_{m=1}^N\prod_{n=1}^N  \frac{1}{\tau \nu_{mn}}
  \distNormal \left( z_{mn}; \, 0, 1 \right)
  \times \bbI[X \in \mcX_\tau],
\end{align*}
where~${z_{mn} = [g_\tau^{-1}(X; \theta)]_{mn}}$ and~$\nu_{mn}$
are the entries of~$V$.
In the zero-temperature limit we recover a discrete
distribution on permutation matrices; otherwise the density
concentrates near the vertices as~${\tau \to 0}$.  This
transformation leverages computationally efficient algorithms
like Sinkhorn-Knopp and the Hungarian algorithm to define a
temperature-controlled variational distribution near the
Birkhoff polytope, and it enjoys many theoretical and practical
benefits.

\begin{figure*}[ht] 
   \centering
   \includegraphics[width=6.5in]{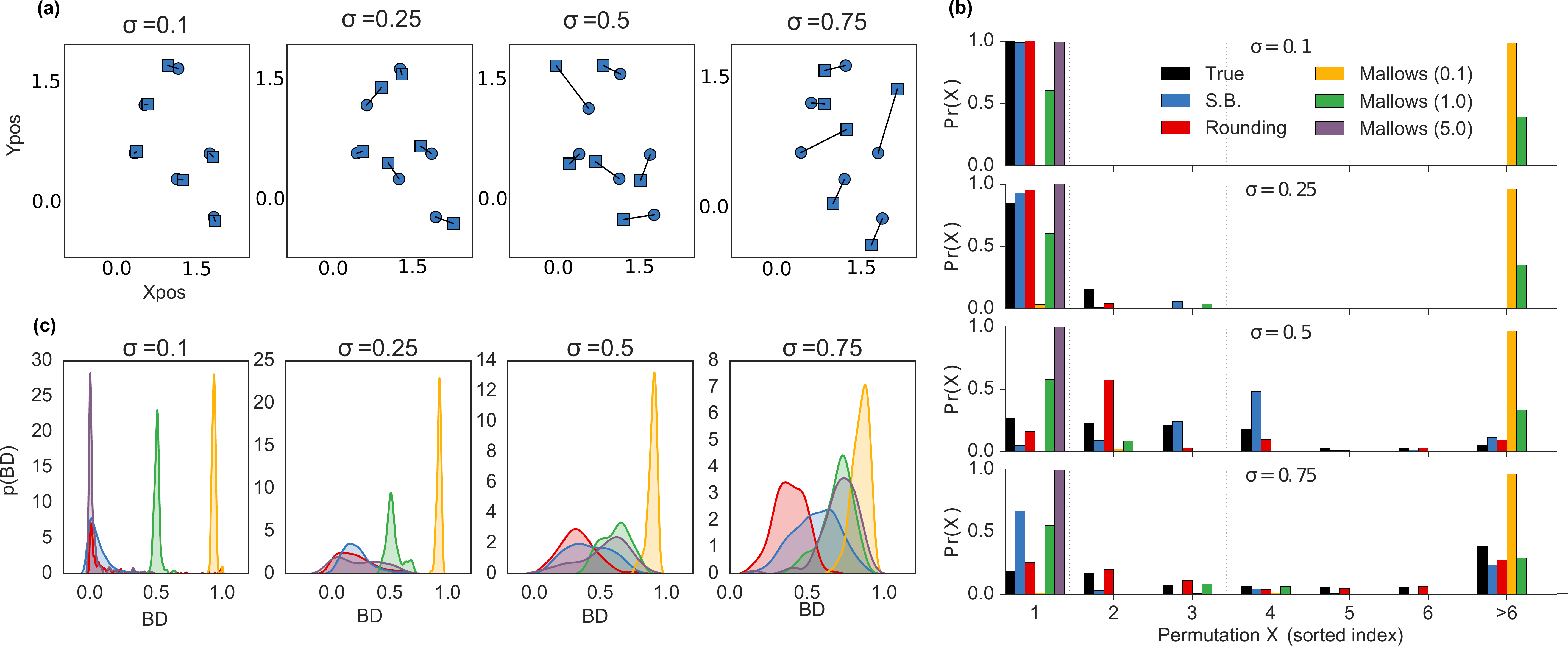}
   \caption{Synthetic matching experiment results. The goal is to
     infer the lines that match squares to circles. (a) Examples of
     center locations (circles) and noisy samples (squares), at
     different noise variances. (b) For illustration, we show the true
     and inferred probability mass functions for different method
     (rows) along with the Battacharya distance (BD) between them for
     a selected case of each $\sigma$ (columns). Permutations
     (indices) are sorted from the highest to lowest actual posterior
     probability. Only the 10 most likely configurations are shown,
     and the 11st bar represents the mass of all remaining
     configurations. (c) KDE plots of Battacharya distances for each
     parameter configuration (based on 200 experiment repetitions) for
     each method and parameter configuration.  For comparison,
     stick-breaking, rounding, and Mallows ($\theta = 1.0$) have BD's
     of .36, .35, and .66, respectively, in the~$\sigma = 0.5$ row of
     (b).}
   \label{fig:synthetic}
\end{figure*}

\subsection{Theoretical considerations}
\label{sec:considerations}

The stick-breaking and rounding transformations introduced above each
have their strengths and weaknesses.  Here we list some of their
conceptual differences.  While these considerations aid in
understanding the differences between the two transformations, the
ultimate test is in their empirical performance, which we study in
Section~\ref{sec:synthetic}.

\begin{itemize}
\item Stick-breaking relaxes to the Birkhoff polytope whereas rounding
  relaxes to~$\reals^{N \times N}$. The Birkhoff polytope is
  intuitively appealing, but as long as the
  likelihood,~$p(y \given X)$, accepts real-valued matrices, either
  may suffice.
  
\item Rounding uses the~$O(N^3)$ Hungarian algorithm in its sampling
  process, whereas stick-breaking has~$O(N^2)$ complexity. In practice,
  the stick-breaking computations are slightly more efficient.
    
\item Rounding can easily incorporate constraints.  If certain
  mappings are invalid, i.e.~${x_{mn} \equiv 0}$, they are given an
  infinite cost in the Hungarian algorithm. This is hard to do this
  with stick breaking as it would change the computation of the upper
  and lower bounds. (In both cases, constraints of the
  form~$x_{mn} \equiv 1$ simply reduce the dimension of the inference
  problem.)
  
\item Stick-breaking introduces a dependence on ordering.  While the
  mapping is bijective, a desired distribution on the Birkhoff polytope
  may require a complex distribution for~$B$.  Rounding, by contrast,
  is more ``symmetric'' in this regard.
  
\end{itemize}

In summary, stick-breaking offers an intuitive advantage---an exact
relaxation to the Birkhoff polytope---but it suffers from its
sensitivity to ordering and its inability to easily incorporate
constraints.  As we show next, these concerns ultimately lead us to
favor the rounding based methods in practice.

\section{Synthetic Experiments}
\label{sec:synthetic}
We are interested in two principal questions: 
 (i) how well can the stick-breaking and rounding re-parameterizations
of the Birkhoff polytope approximate the true posterior distribution
over permutations in tractable, low-dimensional cases? and (ii)
when do our proposed continuous relaxations offer
advantages over alternative  Bayesian permutation
inference algorithms?

% We first studied how stick-breaking and rounding perform in simple
% categorical inference tasks, where they offer an alternative to the
% Gumbel-softmax method.  We found that our methods were comparable,
% though slightly inferior; this is the price paid for techniques that
% extend to more complicated discrete inference problems. Since our
% main interest lies in permutation inference, we defer these results
% to the supplement. 

To assess the quality of our approximations for distributions over
permutations, we considered a toy matching problem in which we are
given the locations of~$N$ cluster centers and a corresponding set
of~$N$ observations, one for each cluster, corrupted by Gaussian
noise.  Moreover, the observations are permuted so there is no
correspondence between the order of observations and the order of the
cluster centers.  The goal is to recover the posterior distribution
over permutations.  For~$N=6$, we can explicitly enumerate
the~$N!=720$ permutations and compute the posterior exactly.
 
As a baseline, we consider the Mallows distribution \cite{Mallows1957}
with density over a permutations $\phi$ given by
$p_{\theta, \phi_0}(\phi)\propto \exp(-\theta d(\phi,\phi_0))$, where
$\phi_0$ is a central permutation,
${d(\phi,\phi_0)=\sum_{i=1}^N |\phi(i)-\phi_0(i)|}$ is a distance
between permutations, and $\theta$ controls the spread around
$\phi_0$. This is the most popular exponential family model for
permutations, but since it is necessarily unimodal, it can fail to
capture complex permutation distributions.

 \begin{table}[h]
  \caption{Mean BDs in the synthetic matching experiment for various methods and observation variances.}
  \label{table:BDs}
  \centering
  \begin{tabular}{lllll}
    & \multicolumn{4}{c}{Variance $\sigma^2$} \\
    \cmidrule(lr){3-4} 
    \textbf{Method} & $.1^2$ & $.25^2$ & $.5^2$ & $.75^2$ \\
    \hline
    Stick-breaking & .09 & .23 & .41 & .55 \\
    Rounding & \textbf{.06} & \textbf{.21}  & \textbf{.32}  & \textbf{.38} \\
    Mallows $(\theta=0.1)$ & .93 & .92 & .89  & .85 \\
    Mallows $(\theta=0.5)$ & .51 & .53  & .61 & .71 \\
    Mallows $(\theta=2)$ & .23 & .33 & .53  & .69 \\
    Mallows $(\theta=5)$ & .08 & .27 & .54 & .72 \\
    Mallows $(\theta=10)$ & .08 & .27 & .54  & .72 \\
    \bottomrule
  \end{tabular}
\end{table}

\begin{figure*}[ht]
  \centering
  \includegraphics[width=6in]{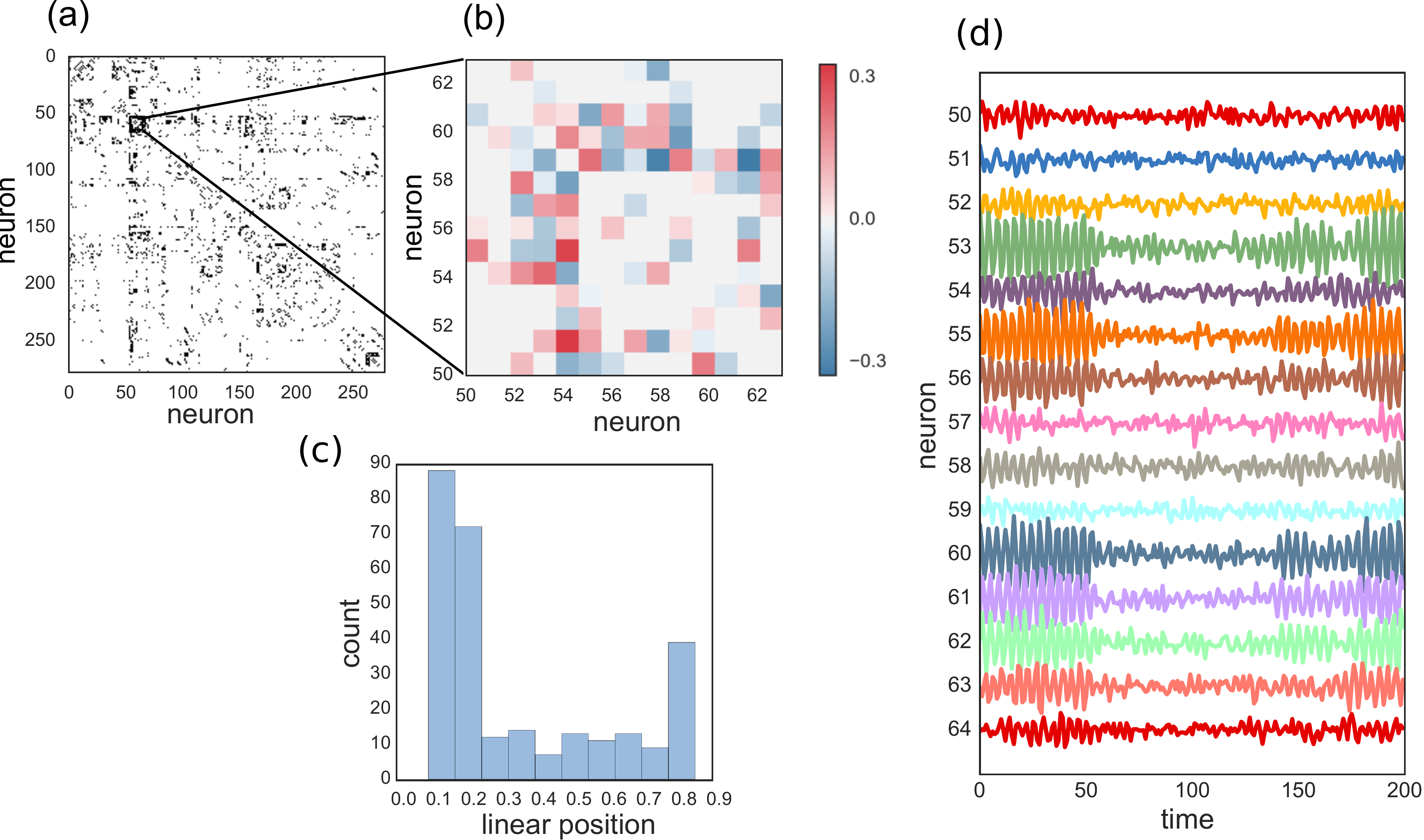} 
  \caption{Problem setup. (a) Hermaphrodite C.elegans reference
    connectome (from \cite{varshney2011structural,wormatlas})
    consisting of 278 somatic neurons, merging two distinct types of
    synapses: chemical and electrical (gap junctions). (b) Example of
    matrix $W$ consistent with the connectome information (only 14
    neurons for visibility), (c) Distribution of neuron position in
    the body, zero means head and one means tail. From
    \cite{white1986structure,wormatlas} (d). Sampled linear dynamical system with matrix $W$.}
  \vspace{-1em}
  \label{fig:connectome}
\end{figure*}

We measured the discrepancy between true posterior and an empirical
estimate of the inferred posteriors using using the Battacharya
distance (BD). We fit $q_\tau(X; \theta)$ with an annealing schedule
for both stick-breaking and rounding transformations, sampled the
variational posterior, and rounded the samples to the nearest
permutation matrix with the Hungarian algorithm. For the Mallows
distribution, we set $\phi_0$ to the MAP estimate, also found with the
Hungarian algorithm, and sampled using MCMC.
 
We found our method outperforms the simple Mallows distribution and
reasonably approximates non-trivial distributions over
permutations. Fig~\ref{fig:synthetic} illustrates our findings,
showing (a) sample experiment configurations; (b) examples of
inferred, discrete, posteriors for stick breaking, rounding, and
Mallows at various levels of noise; and (c) histogram of Battacharya
distance.  The latter are summarized in Table~\ref{table:BDs}.

\section{Inferring neuron identities in \textit{C. elegans}}
\label{sec:celegans}

\begin{figure*}[ht]
  \centering
  \includegraphics[width=6in]{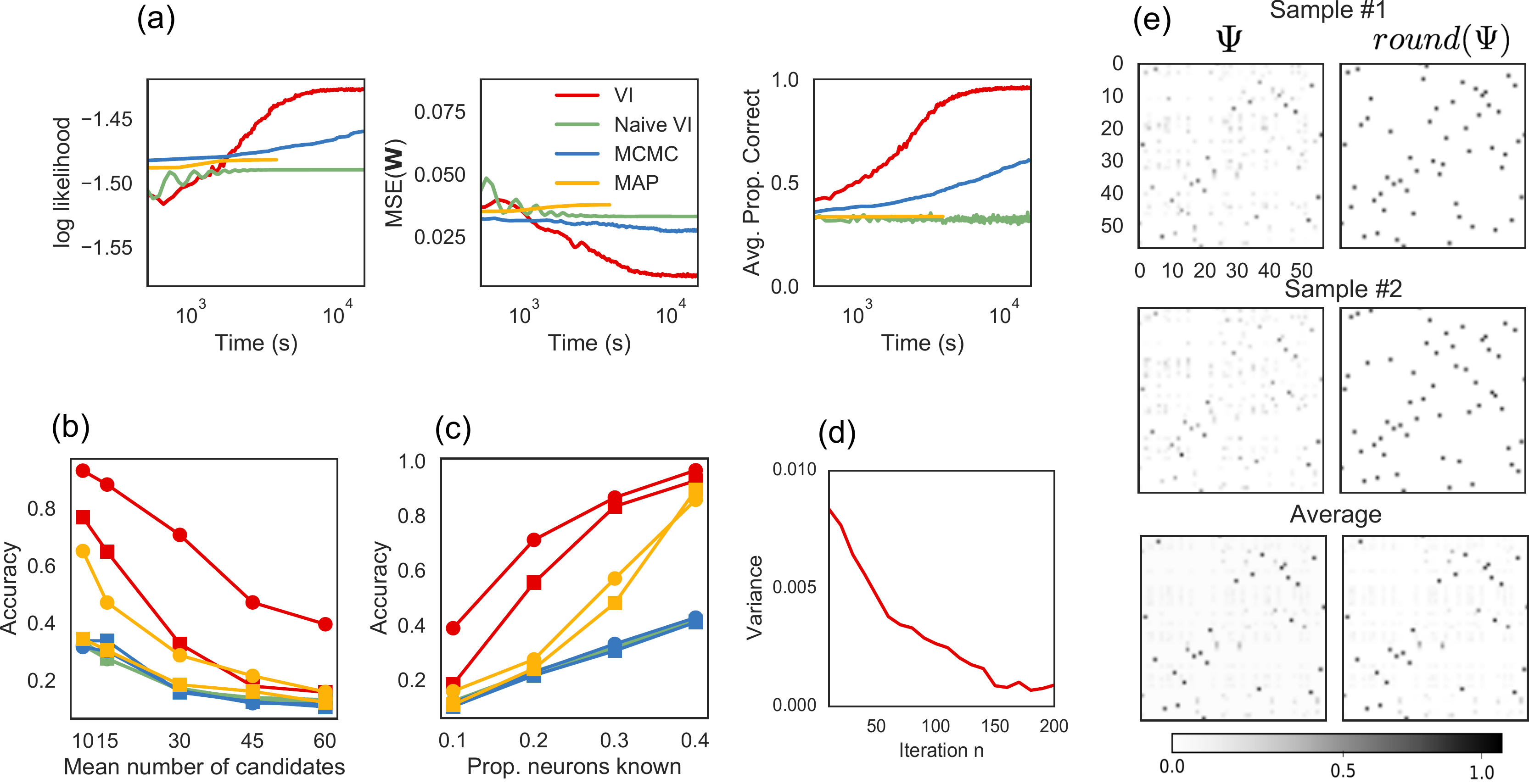} 
  \caption{Results on the C.elegans inference example. (a) An example of convergence of the algorithm, and the baselines. (b) Accuracy of identity inference as a function of mean number of candidates (correlated with $\nu$), for $M=1$ worm (square) and combining information of $M=5$ worms (circles). (c) Accuracy as a function of the proportion of known networks beforehand,  with $\nu=0.1$ (circles) and $\nu=0.05$ (squares). (d)Variance of distribution over permutations (vectorized) as a function of the number of iterations. (e) Two samples of permutation matrices $\mathsf{round}(\Psi)$ (right) and their noisy, non-rounded versions $\Psi$ (left) at the twentieth algorithm iteration. The average of many samples is also shown. Presence of grey dots indicate that the sampling procedure is not deterministic.}
\label{fig:elegantresults}
\end{figure*}

Finally, we consider an application motivated by the study of the
neural dynamics in \textit{C. elegans}. This worm is a model organism
in neuroscience as its neural network is stereotyped from animal to
animal and its complete neural wiring diagram is
known~\citep{varshney2011structural}.  We represent this network, or
connectome, as a binary adjacency
matrix~${A \in \{0,1\}^{N \times N}}$, shown in
Fig.~\ref{fig:connectome}a. The hermaphrodite has~${N=278}$ somatic
neurons, and (undirected) synaptic connections between neurons~$m$
and~$n$ are denoted by~$A_{mn}=1$.

Modern recording technology enables simultaneous measurements of
hundreds of these neurons simultaneously \citep{Kato2015,
  nguyen2016whole}.  However, matching the observed neurons to nodes
in the reference connectome is still a manual task.  Experimenters
consider the location of the neuron along with its pattern of activity
to perform this matching, but the process is laborious and the results
prone to error. We prototype an alternative solution, leveraging the
location of neurons and their activity in a probabilistic model. We
resolve neural identity by integrating different sources of
information from the connectome, some covariates (e.g. position) and
neural dynamics. Moreover, we combine information from many
individuals to facilitate identity resolution.  The hierarchical
nature of this problem and the plethora of prior constraints and
observations motivates our Bayesian approach.

\parhead{Probabilistic Model.}  Let~$J$ denote the number of worms
and~${Y^{(j)} \in \reals^{T_j \times N}}$ denote a recording of
worm~$j$ with~$T_j$ time steps and~$N$ neurons.  We model the neural
activity with a linear dynamical system
${Y^{(j)}_t = {X^{(j)}} W {X^{(j)}}^\trans
  Y^{(j)}_{t-1}+\varepsilon^{(j)}_t}$, where~$\varepsilon_t^{(j)}$ is Gaussian noise.
Here, ~$X^{(j)}$ is a latent permutation of neurons that must be
inferred for each worm in order to align the observations with the
shared dynamics matrix~$W$.  The hierarchical component of the model
is that~$W$ is shared by all worms, and it encodes the influence of
one neuron on another (the rows and columns of~$W$ are ordered in the
same way as the known connectome~$A$). The connectome specifies which
entries of~$W$ may be non-zero: without a connection (${A_{mn}=0}$)
the corresponding weight must be zero; if a connection exists (${A_{mn}=1}$),
we must infer its weight.  Fig.~\ref{fig:connectome}d shows
simulated traces from a network that respects the connectivity of~$A$
and has random Gaussian weights.  The linear model is a simple start;
in future work we can incorporate nonlinear dynamics, more informed
priors on~$W$, etc.

Our goal is to infer~$W$ and~$\{X^{(j)}\}$ given~$\{Y^{(j)}\}$ using
variational permutation inference.  We place a standard Gaussian prior
on~$W$ and a uniform prior on~$X^{(j)}$, and we use the rounding
transformation to approximate the posterior,
$p(W,\{X^{(j)}\} \given \{Y^{(j)}\})\propto p(W) \prod_m p(Y^{(j)} \given W, {X^{(j)}}) \,  p({X^{(j)}}$).

Finally, we use neural position along the worm's body to constrain the
possible neural identities for a given neuron.
We use the known positions of each neuron~\citep{wormatlas}, approximating
the worm as a one-dimensional object with neurons locations distributed
as in Fig.~\ref{fig:connectome}c. Then, given reported positions of the
neurons, we can conceive a binary \textit{constraint} matrix
$C^{(j)}$ so that $C^{(j)}_{mn}=1$ if (observed) neuron $m$ is close enough
to (canonical) neuron $n$; i.e., if their distance is smaller than a
tolerance $\nu$. We enforce this constraint during inference by
zeroing corresponding entries in the parameter matrix $M$ described in
~\ref{sub:rounding}.  This modeling choice greatly reduces the number
parameters of the model, and facilitates inference. 

\parhead{Results.} We compared against three methods: (i) naive
variational inference, where we do not enforce the constraint that
$X^{(j)}$ be a permutation and instead treat each row of~$X^{(j)}$ as
a Dirichlet distributed vector; (ii) MCMC, where we alternate between
sampling from the conditionals of $W$ (Gaussian) and ${X^{(j)}}$, from
which one can sample by proposing local swaps, as described in
\cite{Diaconis2009}, and (iii) maximum a posteriori estimation (MAP).
Our MAP algorithm alternates between the optimizing estimate of $W$ given~$\{X^{(m)}, Y^{(m)}\}$ using linear regression and finding the optimal ${X^{(j)}}$. The second step requires solving a quadratic assignment
problem (QAP) in ${X^{(j)}}$; that is, it can be expressed as
$\mathrm{Tr}(AXBX^\trans)$ for matrices $A,B$. We used the QAP solver
proposed by~\citet{Vogelstein2015}.

We find that our method outperforms each
baseline. Fig.~\ref{fig:elegantresults}a illustrates convergence to a
better solution for a certain parameter configuration. Moreover,
Fig.~\ref{fig:elegantresults}b and Fig.~\ref{fig:elegantresults}c show
that our method outperforms alternatives when there are many possible
candidates and when only a small proportion of neurons are known with
certitude. Fig.~\ref{fig:elegantresults}c also shows that these
Bayesian methods benefit from combining information across many
worms.

Altogether, these results indicate our method enables a more efficient
use of information than its alternatives. This is consistent with
other results showing faster convergence of variational inference over
MCMC \citep{Blei2017}, especially with simple Metropolis-Hastings
proposals. We conjecture that MCMC would eventually obtain similar if
not better results, but the local proposals---swapping pairs of
labels---leads to slow convergence. On the other hand,
Fig~\ref{fig:elegantresults}a shows that our method converges much
more quickly while still capturing a distribution over permutations,
as shown by the overall variance of the samples in Fig~\ref{fig:elegantresults}d and the individual samples in Fig~\ref{fig:elegantresults}e.

\section{Conclusion}

Our results provide evidence that variational permutation
inference is a valuable tool, especially in complex
problems like neural identity inference where information must be
aggregated from disparate sources in a hierarchical model.  As we
apply this to real neural recordings, we must consider more realistic,
nonlinear models of neural dynamics. Here, again, we expect
variational methods to shine, leveraging automatic gradients of the
relaxed ELBO to efficiently explore the space of variational posterior
distributions.

\subsection*{Acknowledgments}
We thank Christian Naesseth for many helpful discussions.  SWL is
supported by the Simons Collaboration on the Global Brain (SCGB)
Postdoctoral Fellowship (418011).
HC is supported by Graphen, Inc.
LP is supported by ARO MURI
W91NF-12-1-0594, the SCGB, DARPA SIMPLEX N66001-15-C-4032, IARPA
MICRONS D16PC00003, and ONR N00014-16-1-2176.  JPC is supported by the
Sloan Foundation, McKnight Foundation, and the SCGB.

\bibliography{refs}
\bibliographystyle{abbrvnat}

\clearpage
\appendix

\input{supp}

\end{document}

%% file: preamble_math.tex
\newcommand{\Noise}{Z}
\newcommand{\noise}{z}

\newcommand{\E}{\mathbb{E}}

\newcommand{\bbH}{\mathbb{H}}

\newcommand{\mcB}{\mathcal{B}}

\newcommand{\mcL}{\mathcal{L}}

\newcommand{\mcX}{\mathcal{X}}

\newcommand{\trans}{\mathsf{T}}

\newcommand{\reals}{\mathbb{R}}
\def\argmax{\operatornamewithlimits{arg\,max}}
\def\argmin{\operatornamewithlimits{arg\,min}}

\newcommand{\distNormal}{\mathcal{N}}

\newcommand{\distBernoulli}{\mathrm{Bern}}

\newcommand{\distCategorical}{\mathrm{Cat}}

\newcommand{\bbI}{\mathbb{I}}
\newcommand{\bbE}{\mathbb{E}}
\newcommand{\bone}{\boldsymbol{1}}

\newcommand{\iid}[1]{\stackrel{\text{iid}}{#1}}

\newcommand{\given}{\, | \,}

\DeclareMathOperator{\diag}{diag}

%% file: preamble_acronyms.tex
\newacronym{KL}{kl}{Kullback-Leibler}
\newacronym{ELBO}{elbo}{\emph{evidence lower bound}}
\newacronym{POPELBO}{pop-elbo}{\emph{population evidence lower bound}}

\newacronym{SVI}{svi}{stochastic variational inference}
\newacronym{BUMPVI}{bump-vi}{bumping variational inference}

\newacronym{GMM}{gmm}{Gaussian mixture model}
\newacronym{LDA}{lda}{latent Dirichlet allocation}

\newacronym{SUTVA}{sutva}{stable unit treatment value assumption}

%% file: supp.tex
\section{Alternative methods of discrete variational inference}

We can gain insight and intuition about the stick-breaking and
rounding transformations by considering their counterparts for
discrete, or categorical, variational inference.  Continuous
relaxations are an appealing approach for this problem, affording
gradient-based inference with the reparameterization trick.
First we review the Gumbel-softmax method~\citep{maddison2016concrete,
  jang2016categorical, kusner2016gans}---a recently proposed
method for discrete variational inference with the reparameterization
trick---then we discuss analogs of our permutation and rounding
transformations for the categorical case.  These can be considered
alternatives to the Gumbel-softmax method, which we compare
empirically in Appendix~\ref{sec:vae}.

Recently there have been a number of proposals for extending the
reparameterization trick~\citep{rezende2014stochastic, Kingma2014} to
high dimensional discrete problems\footnote{Discrete inference is only
  problematic in the high dimensional case, since in low dimensional
  problems we can enumerate the possible values of~$x$ and compute the
  normalizing constant~$p(y) = \sum_x p(y, x)$.} by relaxing them to
analogous continuous problems \citep{maddison2016concrete,
  jang2016categorical, kusner2016gans}.  These approaches are based on
the following observation: if~$x \in \{0,1\}^N$ is a one-hot vector
drawn from a categorical distribution, then the support of~$p(x)$ is
the set of vertices of the~$N-1$ dimensional simplex.  We can
represent the distribution of~$x$ as an atomic density on the simplex.

\subsection{The Gumbel-softmax method}
Viewing~$x$ as a vertex of the simplex motivates a natural relaxation:
rather than restricting ourselves to atomic measures,
consider continuous densities on the simplex. To be concrete, suppose
the density of~$x$ is defined by the transformation,
\begin{align*}
  \noise_n &\iid{\sim} \mathrm{Gumbel}(0, 1) \\
  \psi_n & = \log \theta_n + \noise_n  \\
  x &=  \mathsf{softmax}(\psi / \tau) \\
        &=\left(\frac{e^{\psi_1 / \tau}}{\sum_{n=1}^N e^{\psi_n / \tau}},
      \,\ldots,\,
      \frac{e^{\psi_N / \tau}}{\sum_{n=1}^N e^{\psi_n / \tau}} \right).
\end{align*}
The output~$x$ is now a point on the simplex, and the
parameter~${\theta = (\theta_1, \ldots, \theta_N) \in \reals^N_+}$ can be optimized
via stochastic gradient ascent with the reparameterization trick.

The Gumbel distribution leads to a nicely interpretable model: adding
i.i.d. Gumbel noise to~${\log \theta}$ and taking the argmax yields an
exact sample from the normalized probability mass
function~$\bar{\theta}$,
where~${\bar{\theta}_n = \theta_n / \sum_{m=1}^N
  \theta_m}$~\citep{gumbel1954statistical}. The softmax is a natural
relaxation. As the temperature~$\tau$ goes to zero, the softmax
converges to the argmax function. Ultimately, however, this is just a
continuous relaxation of an atomic density to a continuous density.

Stick-breaking and rounding offer two alternative ways of constructing a
relaxed version of a discrete random variable, and both are amenable
to reparameterization. However, unlike the Gumbel-Softmax, these
relaxations enable extensions to more complex combinatorial objects,
notably, permutations.

\subsection{Stick-breaking}

The stick-breaking transformation to the Birkhoff polytope presented
in the main text contains a recipe for stick-breaking on the simplex.
In particular, as we filled in the first row of the doubly-stochastic
matrix, we were transforming a real-valued vector~$\psi \in \reals^{N-1}$
to a point in the simplex.  We present this procedure for
discrete variational inference again here in simplified form.
Start with a reparameterization of a Gaussian vector,
\begin{align*}
  \noise_n &\iid{\sim} \distNormal(0, 1), \\
  \psi_n & = \mu_n + \nu_n \noise_n, \qquad 1 \leq n \leq N-1,
\end{align*}
parameterized by~${\theta = (\mu_n, \nu_n)_{n=1}^{N-1}}$. 
Then map this to the unit hypercube in a temperature-controlled manner
with the logistic function,
\begin{align*}
  \beta_n &= \sigma(\psi_n / \tau),
\end{align*}
where~${\sigma(u) = (1+e^{-u})^{-1}}$ is the logistic function.
Finally, transform the unit hypercube to a point in the simplex:
\begin{align*}
  x_1 &= \beta_1, \\
  x_n &= \beta_n \left(1- \sum_{m=1}^{n-1} x_m\right), \qquad 2 \leq n \leq N-1,  \\
  x_N &= 1- \sum_{m=1}^{N-1} x_m,
\end{align*}
Here,~$\beta_n$ is the fraction of the remaining ``stick'' of probability
mass assigned to~$x_n$.  This transformation is invertible, the
Jacobian is lower-triangular, and the determinant of the Jacobian is
easy to compute.  \citet{linderman2015dependent} compute the density
of~$x$ implied by a Gaussian density on~$\psi$.

The temperature~$\tau$ controls how concentrated~$p(x)$ is at the
vertices of the simplex, and with appropriate choices of parameters,
in the limit~${\tau \to 0}$ we can recover any categorical
distribution (we will discuss this in detail in
Section~\ref{sec:sblimits}. In the other limit, as~$\tau \to \infty$,
the density concentrates on a point in the interior of the simplex
determined by the parameters, and for intermediate values, the density
is continuous on the simplex.

Finally, note that the logistic-normal construction is only one
possible choice.  We could instead
let~${\beta_n \sim \mathrm{Beta}(\tfrac{a_n}{\tau},
  \tfrac{b_n}{\tau})}$. This would lead to a generalized Dirichlet
distribution on the simplex.  The beta distribution is slightly harder
to reparameterize since it is typically simulated with a rejection
sampling procedure, but~\citet{naesseth2017reparameterization} have
shown how this can be handled with a mix of reparameterization and
score-function gradients.  Alternatively, the beta distribution could
be replaced with the Kumaraswamy
distribution~\citep{kumaraswamy1980generalized}, which is quite
similar to the beta distribution but is easily reparameterizable.

\subsection{Rounding}
Rounding transformations also have a natural analog for discrete
variational inference.  Let~$e_n$ denote a one-hot vector with~$n$-th
entry equal to one.  Define the rounding operator,
\begin{align*}
  \mathsf{round}(\psi)
  &= e_{n^*},
\end{align*}
where
\begin{align*}
  n^* &= \argmin_{n} \|e_n - \psi \|^2 \\
  % &= \argmin_{n} \sum_{m \neq n} \psi_m^2 + (1 - \psi_n)^2  \\
  % &= \argmin_{n} \sum_{m \neq n} \psi_m^2 + \psi_n^2 - 2\psi_n + 1 \\
  % &= \argmin_{n} \|\psi\|^2 - 2\psi_n + 1 \\
  &= \argmax_n \psi_n.
\end{align*}
In the case of a tie, let~$n^*$ be the smallest index~$n$ such
that~$\psi_n > \psi_m$ for all~$m < n$. Rounding effectively
partitions the space into~$N$ disjoint ``Voronoi'' cells,
\begin{align*}
  V_n &= \Big \{ \psi \in \reals^N : \,
        \psi_n \geq \psi_m \, \forall m \; \wedge \;
        \psi_n > \psi_m \, \forall m < n
        \Big \}.
\end{align*}
By definition,~${\mathsf{round}(\psi) = e_{n^*}}$ for
all~${\psi \in V_{n^*}}$

We define a map that pulls points toward their rounded values,
\begin{align}
  \label{eq:round}
  x &=  \tau \psi + (1-\tau) \mathsf{round}(\psi).
\end{align}

\begin{proposition}
  \label{prop:round}
  For~${\tau \in [0,1]}$, the map defined by~\eqref{eq:round} moves
  points strictly closer to their rounded values so
  that~$\mathsf{round}(\psi) = \mathsf{round}(x)$.
\end{proposition}

\begin{proof}
Note that the Voronoi cells are intersections of halfspaces and, as
such, are convex sets.  Since~$x$ is a convex combination of~$\psi$
and~$e_{n^*}$, both of which belong to the convex set~$V_{n^*}$,
$x$~must belong to~$V_{n^*}$ as well.
\end{proof}

Similarly,~$x$ will be a point on the simplex if an only if~$\psi$ is
on the simplex as well.  By analogy to the rounding transformations
for permutation inference, in categorical inference we use a Gaussian
distribution~${\psi \sim \distNormal(\mathsf{proj}(m), \nu)}$,
where~$\mathsf{proj}(m)$ is the projection of~$m \in \reals_+^N$ onto
the simplex.  Still, the simplex has zero measure under the Gaussian
distribution.  It follows that the rounded points~$x$ will almost
surely not be on the simplex either.  The supposition of this approach
is that this is not a problem: relaxing to the simplex is nice but not
required.

In the zero-temperature limit we obtain a discrete distribution on the
vertices of the simplex.  For~${\tau \in (0,1]}$ we have a
distribution on~${\mcX_\tau \subseteq \reals^N}$, the subset of the
reals to which the rounding operation maps. (For~${0 \leq \tau < 1}$
this is a strict subset of~$\reals^N$.) To derive the
density~$q(x)$, we need the inverse transformation and the
determinant of its Jacobian.  From Proposition~\ref{prop:round}, it
follows that the inverse transformation is given by,
\begin{align*}
  \psi &= \frac{1}{\tau} x - \frac{1 - \tau}{\tau} \mathsf{round}(x).
\end{align*}
As long as~$\psi$ is in the interior of its Voronoi cell,
the~$\mathsf{round}$ function is piecewise constant and the
Jacobian is~${\tfrac{\partial\psi}{\partial x} = \tfrac{1}{\tau} I}$,
and its determinant is~$\tau^{-N}$. Taken together, we have,
\begin{multline*}
  q(x; m, \nu) =  \\
  \tau^{-N} \distNormal \left(\frac{1}{\tau}x - \frac{1-\tau}{\tau} \mathsf{round}(x); \, \mathsf{proj}(m), \diag(\nu) \right) \\
  \times \bbI[x \in \mcX_\tau].
\end{multline*}
Compare this to the density of the rounded random variables for
permutation inference. 

\subsection{Limit analysis for stick-breaking}
\label{sec:sblimits}
We show that stick-breaking for discrete variational inference can
converge to any categorical distribution in the zero-temperature
limit.
% To do so, we first show that
% in the zero-temperature limit, the distribution
% of~${\sigma(\psi_n / \tau)}$ converges to a Bernoulli distribution.
% The we show that when~$\sigma(\psi_n / \tau)$ is Bernoulli (rather
% than a continuous density on the unit interval), the distribution
% on~$x$ obtained by applying the stick-breaking transformation
% to~$\psi$ is categorical.

% \begin{proposition}
%   \label{prop:bernoulli}
  Let~${\beta=\sigma(\psi / \tau)}$
  with~${\psi\sim\mathcal{N}(\mu,\nu^2)}$.  In the
  limit~${\tau \to
    0}$ we have~${\beta \sim \distBernoulli(\Phi(-\tfrac{\mu}{\nu}))}$,
  where~$\Phi(\cdot)$ denotes the Gaussian cumulative distribution
  function (cdf). 
% \end{proposition}
  % \begin{proof}
  % To see this, let~$F_\beta$ be the cdf of the random
  % variable~$\beta$. Since~$\beta$ is a random variable on the unit interval,
  % $F_\beta$ is a non-decreasing function on~$[0,1]$ with~${F_\beta(0)=0}$
  % and~${F_\beta(1)=1}$.  Reparameterize~${\psi = \mu + \nu \noise}$
  % where~${\noise \sim \distNormal(0,1)}$. Then we have,
  % \begin{align*}
  %   F_\beta(u) &= \Pr(\sigma(\psi / \tau) < u) \\
  %        &= \Pr(\psi < \tau \sigma^{-1}(u)) \\
  %        &= \Pr(\noise < \tfrac{\tau}{\nu} \sigma^{-1}(u) - \tfrac{\mu}{\nu}))\\
  %        &= \Phi(-\tfrac{\tau}{\nu}\sigma^{-1}(u) - \tfrac{\mu}{\nu}).
  % \end{align*}
  % By the continuity of~$\Phi$ we have,
  % \begin{align*}
  %   \lim_{\tau \to 0} F_\beta(u) &= \Phi(-\tfrac{\mu}{\nu}) &
  %   \text{for } u &\in (0,1).
  % \end{align*}
  % This is the cdf of a Bernoulli random with
  % probability~${\rho = \Phi(-\tfrac{\mu}{\nu})}$.
  % \end{proof}
  %
% \begin{proposition}
%   \label{prop:categorical}
%   As above, let~${\beta_n=\sigma(\psi_n / \tau)}$.
  Moreover, when~${\beta_n \sim \distBernoulli(\rho_n)}$
  with~${\rho_n \in [0,1]}$ for~${n=1, \ldots, N}$, the random
  variable~$x$ obtained from applying the stick-breaking
  transformation to~$\beta$ will have an atomic distribution with atoms
  in the vertices of $\Delta_{N}$; i.e,
  ${x \sim \distCategorical(\pi)}$ where
  \begin{align*}
    \pi_1 &= \rho_1 \\
    \pi_n &=  \rho_n \prod_{m=1}^{n-1} (1-\rho_m)  \qquad
          n=2, \ldots, N-1, \\
    \pi_N &= \prod_{m=1}^{N-1} (1-\rho_m).
\end{align*}
% \end{proposition}

% \begin{proof}
%   From the stick-breaking definition,~${x_1 = \beta_1}$,
%   ${x_n = \beta_n (1- \sum_{m < n} x_m)}$,
%   and~${x_N = 1-\sum_{m < N} x_m}$.
%   When~${\beta_n \in \{0,1\}}$ for all~${n = 1, \ldots, N-1}$,
%   we have the following equivalencies. For the first element,
%   \begin{align*}
%     x_1 = 1 &\iff \beta_1 = 1;
%   \end{align*}
%   for~${1 < n < N-1}:$
%   \begin{align*}
%     x_n = 1 &\iff (\beta_n = 1) \bigwedge_{m=1}^{n-1} (\beta_m = 0);
%   \end{align*}
%   and for the last element,
%   \begin{align*}
%     x_N = 1 &\iff \bigwedge_{m=1}^{N-1} (\beta_m = 0).
%   \end{align*}
%   These events are mutually exclusive, implying that~$x$ will
%   necessarily be a one-hot vector, i.e. a categorical random variable.
%   Since~$\beta_1, \ldots, \beta_{N-1}$ are independent Bernoulli random
%   variables, the probabilities of these events are given
%   by the~${\pi, \ldots, \pi_N}$ stated in the proposition. 
% \end{proof}

These two facts, combined with the invertibility of the
stick-breaking procedure, lead to the following proposition

\begin{proposition}
  In the zero-temperature limit, stick-breaking of logistic-normal
  random variables can realize any categorical distribution on~$x$.
\end{proposition}

\begin{proof}
  There is a one-to-one correspondence between~${\pi \in \Delta_N}$
  and~${\rho \in [0,1]^{N-1}}$.  Specifically,
  \begin{align*}
    \rho_1 &= \pi_1 \\
    \rho_n &= \frac{\pi_n}{\prod_{m=1}^{n=1} 1-\rho_m}
             \quad \text{for } n = 2, \ldots, N-1.
  \end{align*}
  Since these are recursively defined, we can substitute the
  definition of~$\rho_m$ to obtain an expression for~$\rho_n$ in terms
  of~$\pi$ only.  Thus, any desired categorical distribution~$\pi$
  implies a set of Bernoulli parameters~$\rho$.  In the zero
  temperature limit, any desired~$\rho_n$ can be obtained with
  appropriate choice of Gaussian mean~$\mu_n$ and
  variance~$\nu_n^2$. Together these imply that stick-breaking can
  realize any categorical distribution when~${\tau \to 0}$.
\end{proof}

\subsection{Variational Autoencoders (VAE) with categorical latent variables}
\label{sec:vae}

We considered the density estimation task on MNIST digits, as in
\citet{maddison2016concrete, jang2016categorical}, where observed
digits are reconstructed from a latent discrete code. We used the
continuous ELBO for training, and evaluated performance based on the
marginal likelihood, estimated with the variational
objective of the discretized model.  We compared against the methods of
\cite{jang2016categorical, maddison2016concrete} and obtained the
results in Table~\ref{tab:vae}.  While stick-breaking and rounding
fare slightly worse than the Gumbel-softmax method, they are readily
extensible to more complex discrete objects, as shown in the main
paper.

\begin{table}[h]
  \caption{Summary of results in VAE}
  \label{tab:vae}
  \centering
  \begin{tabular}{ll}
    \textbf{Method} & $- \log p(x)$ \\
    \hline
    Gumbel-Softmax    & 106.7 \\
    Concrete  &  111.5\\
    Rounding &  121.1 \\
    Stick-breaking & 119. 8\\
    \bottomrule
  \end{tabular}
\end{table}

% \begin{table*}[t]
%   \caption{Battacharya distances in the synthetic matching experiment}
%   \label{sample-table}
%   \centering
%   \begin{tabular}{llllllll}
%    & \multicolumn{1}{c}{Rounding} & \multicolumn{1}{c}{Stick-breaking} & \multicolumn{5}{c}{Mallows}\\
%     \cmidrule(lr){2-2} \cmidrule(lr){3-3} \cmidrule(lr){4-8}
%     &    & &   $\theta=0.1$ &  $\theta=1$ & $\theta=2$ & $\theta=5$ & $\theta=10$ \\
%     \midrule
%     $\sigma=0.1$     & .06 & .09  &.93 &.51& .23  & .08 &.08\\
%     $\sigma=0.25$     & .21 & .23 & .92 &.53 & .33&  .27 &.27\\
%      $\sigma=0.5$     & .32 & .41 & .89 &.61 & .53&  .54& .54\\
%      $\sigma=0.75$     & .38   & .55 & .85 &.71 & .69&  .72 &.72\\
   
%     \bottomrule
%   \end{tabular}
% \end{table*}

Figure~\ref{fig:VAE} shows MNIST reconstructions using Gumbel-Softmax,
stick-breaking and rounding reparameterizations. In all the three
cases reconstructions are reasonably accurate, and there is diversity
in reconstructions.
\begin{figure*}[t]
  \centering
  \includegraphics[width=5.in]{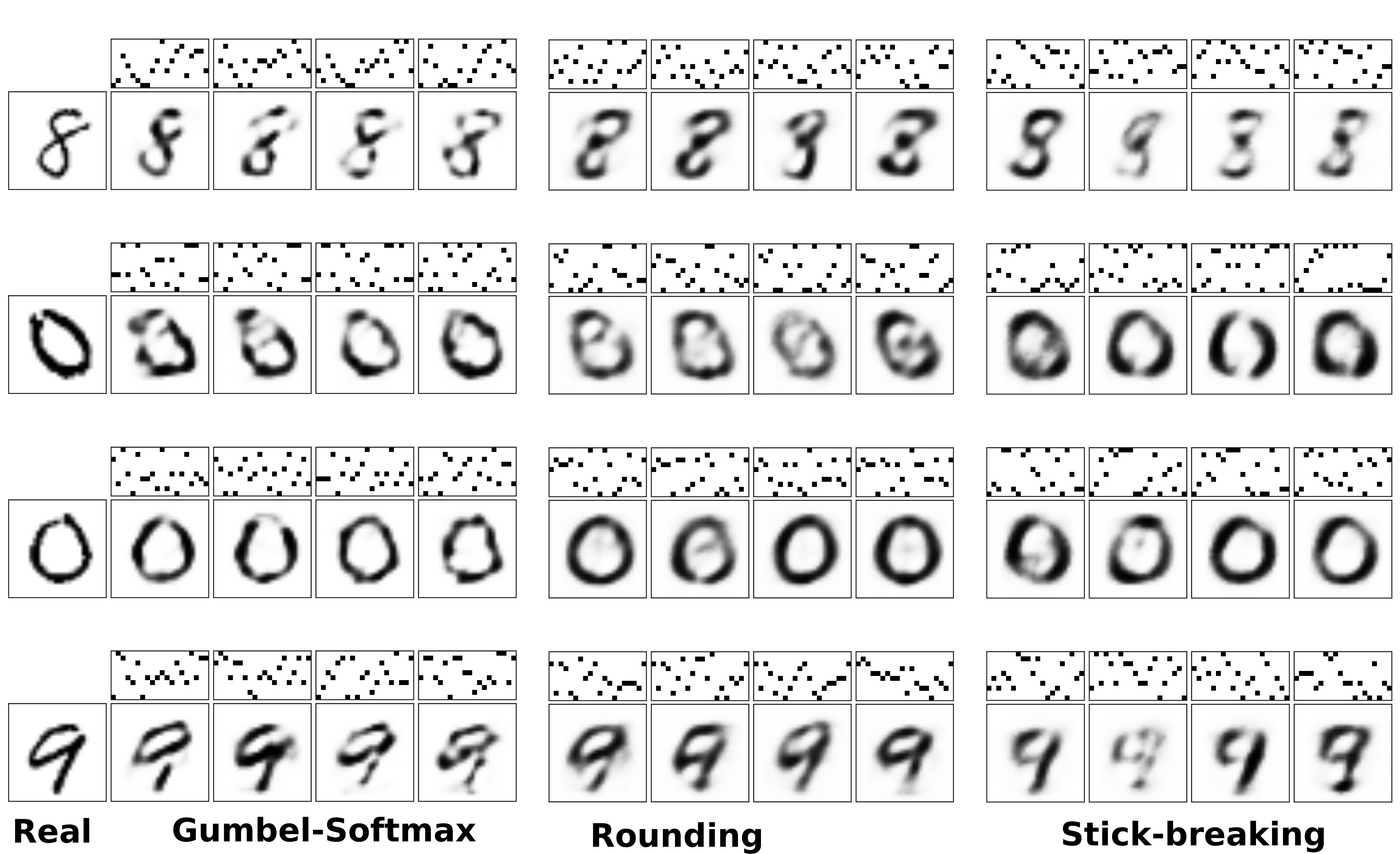} 
  \caption{\textit{Examples of true and reconstructed digits from their
    corresponding discrete latent variables.} The real input image is
    shown on the left, and we show sets of four samples from the
    posterior predictive distribution for each discrete variational
    method: Gumbel-softmax, rounding, and stick-breaking.  Above each
    sample we show the corresponding sample of the discrete latent
    ``code.''  The random codes consist of of~$K=20$ categorical
    variables with $N=10$ possible values each.  The codes are shown
    as~${10 \times 20}$~binary matrices above each image.}
\label{fig:VAE}
\end{figure*}

\section{Variational permutation inference details}
\label{sec:details}

Here we discuss more of the subtleties of variational permutation
inference and present the mathematical derivations in more detail. 

\subsection{Continuous prior distributions.} 
Continuous relaxations require re-thinking the objective: the model
log-probability is defined with discrete latent variables, but our
relaxed posterior is a continuous density. As in
\cite{maddison2016concrete}, we instead maximize a relaxed ELBO.  We
assume the functional form of the likelihood remains unchanged, and
simply accepts continuous values instead of discrete. However, we need
to specify a new continuous prior $p(X)$ over the relaxed discrete
latent variables, here, over relaxations of permutation matrices. It
is important that the prior be sensible: ideally, the prior should
penalize values of~$X$ that are far from permutation matrices.

For our categorical experiment on MNIST we use a mixture of Gaussians
around each
vertex,~${p(x) = \tfrac{1}{N} \sum_{n=1}^N \mathcal{N}(x \given e_k,
  \eta^2)}$.  This can be extended to permutations, where we use a
mixture of Gaussians for each coordinate,
\begin{align}
\label{eq:permprior}
  p(X) &= \prod_{m=1}^N \prod_{n=1}^N
  \frac{1}{2} \left(\mathcal{N}(x_{mn} \given 0, \eta^2) + \distNormal(x_{mn} \given 1, \eta^2 \right).
\end{align}
Although this prior puts significant mass around invalid points
(e.g.~${(1, 1, \ldots, 1)}$), it penalizes~$X$ that are far from~$\mcB_N$.

\subsection{Computing the ELBO}
Here we show how to evaluate the ELBO.  Note that the stick-breaking
and rounding transformations are compositions of invertible
functions,~${g_\tau = h_\tau \circ f}$ with ${\Psi = f(\noise; \theta)}$ and
${X = h_\tau(\Psi)}$.  In both cases,~$f$ takes in a matrix of
independent standard Gaussians~$(\noise)$ and transforms it with the
means and variances in~$\theta$ to output a matrix~$\Psi$ with
entries~${\psi_{mn} \sim \distNormal(\mu_{mn}, \nu^2_{mn})}$.
Stick-breaking and rounding differ in the temperature-controlled
transformations~$h_\tau(\Psi)$ they use to map~$\Psi$ toward the
Birkhoff polytope.

To evaluate the ELBO, we must compute the density
of~$q_\tau(X; \theta)$.
Let~${J_{h_\tau}(u) = \frac{\partial h_\tau(U)}{\partial U}
  \big|_{U=u}}$ denote the Jacobian of a function~$h_\tau$ evaluated
at value~$u$. By the change of variables theorem and properties of the
determinant,
\begin{align*}
  q_\tau(X; \theta)
  &= p \big(h_\tau^{-1}(X) ;\theta \big)
    \times \big| J_{h_\tau^{-1}}(X) \big|
  \\
  &= p \big(h_\tau^{-1}(X); \theta \big)
    \times \big| J_{h_\tau}(h_\tau^{-1}(X)) \big|^{-1}.
\end{align*}
Now we appeal to the law of the unconscious statistician to compute
the entropy of~$q_\tau(X; \theta)$,
\begin{align}
  \label{eq:elbo2}
  \nonumber \E_{q_\tau(X; \theta)} & \Big[- \log q(X; \theta) \Big] \\
  \nonumber &= \E_{p(\Psi;\theta)}
              \Big[ - \log p(\Psi;\theta) +
              \log \left| J_{h_\tau}(\Psi) \right| \Big] \\
  &= \bbH(\Psi; \theta)  +
  \E_{p(\Psi;\theta)} \Big[ \left| J_{h_\tau}(\Psi) \right| \Big].
\end{align}
Since~$\Psi$ consists of independent Gaussians with
variances~$\nu_{mn}^2$, the entropy is simply,
\begin{align*}
  \bbH(\Psi; \theta) &=  \frac{1}{2} \sum_{m,n} \log(2 \pi e \nu_{mn}^2).
\end{align*}
We estimate the second term of equation \eqref{eq:elbo2} using
Monte-Carlo samples. For both transformations, the Jacobian has a
simple form.

\paragraph{Jacobian of the stick-breaking transformation.}
Here~$h_\tau$ consists of two steps:
map~${\Psi \in \reals^{N-1 \times N-1}}$
to~$B \in [0,1]^{N-1 \times N-1}$ with a temperature-controlled,
elementwise logistic function, then map~$B$ to~$X$ in the Birkhoff
polytope with the stick-breaking transformation.

As with the standard stick-breaking transformation to the simplex,
our transformation to the Birkhoff polytope is feed-forward;
i.e. to compute~$x_{mn}$ we only need to know the values of~$\beta$
up to and including the~$(m,n)$-th entry. Consequently, the
Jacobian of the transformation is triangular, and its determinant
is simply the product of its diagonal.

We derive an explicit form in two steps. With a slight abuse of
notation, note that the Jacobian of~$h_\tau(\Psi)$ is given
by the chain rule,
\begin{align*}
  J_{h_\tau}(\Psi)
  &= \frac{\partial X}{\partial  \Psi}
    = \frac{\partial X}{\partial B} \frac{\partial B}{\partial \Psi}.
\end{align*}
Since both transformations are bijective, the determinant is,
\begin{align*}
  \big| J_{h_\tau}(\Psi) \big|
  &= \left| \frac{\partial X}{\partial B} \right| \,
     \left| \frac{\partial B}{\partial \Psi} \right|.
\end{align*}
the product
of the individual determinants.  The first determinant is,
\begin{align*}
  \left| \frac{\partial X}{\partial B} \right|
  &= \prod_{m=1}^{N-1} \prod_{n=1}^{N-1} \frac{\partial x_{mn} }{\partial {\beta}_{mn}} 
  = \prod_{m=1}^{N-1} \prod_{n=1}^{N-1} (u_{mn} - \ell_{mn}).
\end{align*}
The second transformation, from~$\Psi$ to~$B$, is an element-wise,
temperature-controlled logistic transformation such that,
\begin{align*}
  \left| \frac{\partial B}{\partial \Psi} \right| 
  &= \prod_{m=1}^{N-1} \prod_{n=1}^{N-1} \frac{\partial \beta_{mn}}{\partial \psi_{mn}} \\
  &= \prod_{m=1}^{N-1} \prod_{n=1}^{N-1}
    \frac{1}{\tau} \sigma \left(\psi_{mn} / \tau \right)
    \sigma \left(-\psi_{mn} / \tau \right).
\end{align*}

% \begin{align*}
%   &= \prod_{i=1}^{N-1} \prod_{j=1}^{N-1} \frac{\partial}{\partial {x}_{ij}}
%     \sigma^{-1} \left( \frac{{x}_{ij} - \ell_{ij}}{u_{ij} - \ell_{ij}} \right ) \\
%   &= \prod_{i=1}^{N-1} \prod_{j=1}^{N-1}
%     \left( \frac{1}{u_{ij} - \ell_{ij}} \right )
%     \left( \frac{u_{ij} - \ell_{ij}}{{x}_{ij} - \ell_{ij}} \right )
%     \left( \frac{u_{ij} - \ell_{ij}}{u_{ij} - {x}_{ij}} \right ) \\
%   &= \prod_{i=1}^{N-1} \prod_{j=1}^{N-1}
%     \frac{u_{ij} - \ell_{ij}}{({x}_{ij} - \ell_{ij}) (u_{ij} - {x}_{ij})}
% \end{align*}

% To compute the gradient of the forward transformation $h$, one simply needs to invert the above (or put a negative sign, in the logarithm scale). Finally,  to incorporate the effect of $\sigma$ ($B=\sigma(\Psi)$), by the chain rule,  one only needs to add a term corresponding to this derivative, $d\sigma(x)/dx=\sigma(x)\sigma(-x)$. 

It is important to note that the transformation that maps
$B \rightarrow X$ is only piecewise continuous: the function is not
differentiable at the points where the bounds change; for example,
when changing~$B$ causes the active upper bound to switch from the row
to the column constraint or vice versa.  In practice, we find
that our stochastic optimization algorithms still perform reasonably
in the face of this discontinuity.

\paragraph{Jacobian of the rounding transformation.}
The rounding transformation is given in matrix form
in the main text, and we restate it here in coordinate-wise form
for convenience,
\begin{align*}
  x_{mn} = [h_\tau(\Psi)]_{mn} &= \tau \psi_{mn} + (1-\tau) [\mathsf{round}(\Psi)]_{mn}.
\end{align*}
This transformation is piecewise linear with jumps at the boundaries
of the ``Voronoi cells;'' i.e., the points where~$\mathsf{round}(X)$
changes. The set of discontinuities has Lebesgue measure zero so the
change of variables theorem still applies.  Within each Voronoi cell,
the rounding operation is constant, and the Jacobian is,
\begin{align*}
  \log \big| J_{h_\tau}(\Psi) \big| = \sum_{m,n} \log \tau = N^2 \log \tau.
\end{align*}
For the rounding transformation with given temperature, the Jacobian
is constant.

\section{Experiment details}
We used Tensorflow \citep{Abadi2016} for the VAE experiments, slightly
changing the code made available from \cite{jang2016categorical}. For
experiments on synthetic matching and the C. elegans example we used
Autograd \citep{maclaurin2015autograd}, explicitly avoiding
propagating gradients through the non-differentiable~$\mathsf{round}$
operation, which requires solving a matching problem.

We used ADAM~\citep{kingma2014adam} with learning rate 0.1 for
optimization. For rounding, the parameter vector $V$ defined in
\ref{sub:rounding} was constrained to lie in the interval
$[0.1, 0.5]$. Also, for rounding, we used ten iterations of the
Sinkhorn-Knopp algorithm, to obtain points in the Birkhoff
polytope. For stick-breaking the variances $\nu$ defined in
\ref{sub:stickbreaking} were constrained between $10^{-8}$ and~$1$. In
either case, the temperature, along with maximum values for the noise
variances were calibrated using a grid search.
 
In the C. elegans example we considered the symmetrized version of the
adjacency matrix described in \citep{varshney2011structural}; i.e. we
used $A'=(A+A^\top)/2$, and the matrix $W$ was chosen antisymmetric,
with entries sampled randomly with the sparsity pattern dictated by
$A'$. To avoid divergence, the matrix $W$ was then re-scaled by 1.1
times its spectral radius. This choice, although not essential,
induced a reasonably well-behaved linear dynamical system, rich in
non-damped oscillations. We used a time window of $T=1000$ time
samples, and added spherical standard noise at each time. All results
in Figure \ref{fig:elegantresults} are averages over five experiment
simulations with different sampled matrices $W$. For results in Figure
\ref{fig:elegantresults}b we considered either one or four worms
(squares and circles, respectively), and for the x-axis we used the values
$\nu \in \{0.0075,0.01,0.02,0.04,0.05\}$. We fixed the number of known neuron
identities to 25 (randomly chosen). For results in Figure
\ref{fig:elegantresults}c we used four worms and considered two values
for $\nu$; 0.1 (squares) and 0.05 (circles). Different x-axis values
correspond to fixing 110, 83, 55 and 25 neuron identities.